\documentclass[11pt]{article}   	
\usepackage{graphicx}				
								
\usepackage[usenames]{color}
\definecolor{plum}  {rgb}{.4,0,.4}
\definecolor{BrickRed} {rgb}{0.6,0,0}
								
\usepackage{amsmath,amssymb,amsthm}
\usepackage{enumerate}
\usepackage{mathrsfs}
\usepackage{verbatim}
\usepackage{diagbox}
\usepackage{enumerate}
\usepackage[plainpages=false,pdfpagelabels,colorlinks=true,linkcolor=BrickRed,citecolor=blue]{hyperref}
\usepackage{mdframed}
\usepackage{lipsum}
\usepackage{thmtools,thm-restate}
\usepackage{fullpage}
\usepackage{marginnote} 
\usepackage[round]{natbib}
\newmdtheoremenv{theo}{Theorem}
\usepackage{scalerel}[2014/03/10]
\usepackage[usestackEOL]{stackengine}
\def\dashint{\,\ThisStyle{\ensurestackMath{%
  \stackinset{c}{.2\LMpt}{c}{.5\LMpt}{\SavedStyle-}{\SavedStyle\phantom{\int}}}%
  \setbox0=\hbox{$\SavedStyle\int\,$}\kern-\wd0}\int}
\theoremstyle{definition}
\newtheorem{defn}{Definition}[section]
\newtheorem{prop}[defn]{Proposition}

\newtheorem{mythm}[defn]{Theorem}
\newtheorem{lem}[defn]{Lemma}

\newtheorem{remark}{Remark}
{\theoremstyle{plain}

}
\newmdtheoremenv{lembox}[defn]{Lemma}
\usepackage{CJK}

\newcommand{\mrm}[1]{\mathrm{#1}}
\newcommand{\argmin}[1]{\underset{#1}{\mrm{argmin}} \ }
\newcommand{\algo}{\widehat{f}}
\newcommand{\cH}{\mathcal{H}}
\newcommand{\norm}[1]{\left\|#1\right\|}

\newcommand{\inner}[1]{\left\langle #1 \right\rangle}
\newcommand{\reals}{{\mathbb R}}
\def\deq{\triangleq}

\title{Consistency of Interpolation with Laplace Kernels is a High-Dimensional Phenomenon}
\author{Alexander Rakhlin\\MIT \and Xiyu Zhai\\MIT}
\date{}							

\begin{document}
\maketitle
\thispagestyle{empty}

\begin{abstract}
	We show that minimum-norm interpolation in the Reproducing Kernel Hilbert Space corresponding to the Laplace kernel is not consistent if input dimension is constant. The lower bound holds for any choice of kernel bandwidth, even if selected based on data. The result supports the empirical observation that minimum-norm interpolation (that is, exact fit to training data) in RKHS generalizes well for some high-dimensional datasets, but not for low-dimensional ones.
\end{abstract}

\section{Introduction}

Can a method perfectly fit the training data perform well out-of-sample? In the last few years, this question was raised in the context of over-parametrized neural networks \citep{zhang2016understanding,belkin2018understand}, kernel methods \citep{belkin2018understand,liang2018just}, and local nonparametric rules \citep{belkin2018overfitting,belkin2018does}. Experiments on a range of real and synthetic datasets confirm that procedures attaining zero training error do not necessarily overfit and can generalize well \citep{wyner2017explaining,zhang2016understanding,belkin2018understand,liang2018just}. In particular, Kernel Ridge Regression
\begin{align}
	\label{def:rkhs}
	\algo ~\in~ \argmin{f\in\cH} \frac{1}{n}\sum_{i=1}^n (f(x_i)-y_i)^2 + \lambda\norm{f}^2_{\cH}
\end{align}
performs ``unreasonably well'' in the regime $\lambda= 0$, even though the solution (generally) interpolates the data. Here $\cH$ is a Reproducing Kernel Hilbert Space (RKHS) corresponding to a kernel $K$, $\norm{\cdot}_{\cH}$ is the corresponding RKHS norm, and $(x_1,y_1),\ldots,(x_n,y_n)\in \reals^d\times \reals$ are the training data. Since the argmin in \eqref{def:rkhs} is not unique when $\lambda=0$, we consider the minimum-norm interpolating solution
\begin{align}
	\label{def:rkhs_interpolation}
	&\argmin{f\in\cH} ~~\norm{f}_{\cH}\\
	&\text{s.t.}~~ f(x_i)=y_i,~~ i=1,\ldots,n \notag
\end{align}

The conditions under which interpolation, such as Kernel ``Ridgeless'' Regression, performs well are poorly understood.  \citep{liang2018just} studied the high-dimensional regime $n\asymp d$, explicating (under additional assumptions) a phenomenon of implicit regularization, due to the curvature of the kernel function, high dimensionality, and favorable geometric properties of the training data, as quantified by the spectral decay of the kernel and covariance matrices.

The mechanism of implicit regularization in \citep{liang2018just} relies on high dimensionality $d$ of the input space, and it is unclear whether such a high dimensionality is necessary for good out-of-sample performance of interpolation. Perhaps there is a different mechanism that leads to generalization of minimum-norm interpolants \eqref{def:rkhs_interpolation} for any dimensionality of the input space? Our experiments suggest that this is not the case: \emph{minimum-norm interpolant does not appear to perform well in low dimensions}. The present paper provides a theoretical justification for this observation. We show that the estimation error of \eqref{def:rkhs_interpolation} with the Laplace kernel does not converge to zero as the sample size $n$ increases, unless $d$ scales with $n$. 

We chose to study the Laplace kernel 
\begin{align}
	\label{def:laplace_kernel}
	K_c(x,x')= c^d e^{-c\norm{x-x'}}
\end{align}
for several reasons. First, \cite{belkin2018understand} argue that Laplace kernel regression is more similar to ReLU neural networks than, for instance, Gaussian kernel regression. More precisely, the nonlinearities introduced by the Laplace kernel allow SGD to have a large ``computational reach'' (as argued in \citep{belkin2018understand}, the number of epochs required to fit natural vs random labels for Laplace kernel is well-aligned with the corresponding behavior in ReLU networks). Second, for large $c$, the minimum-norm interpolant in $d=1$ corresponds to simplicial interpolation of \cite{belkin2018overfitting}, and it may be possible to borrow some of the intuition from the latter paper for higher dimensions. Finally, the RKHS norm corresponding to Laplace kernel can be related to a Sobolev norm, facilitating the development of the lower bound in this paper. We also note that non-differentiability of the kernel function at $0$ puts it outside of the assumptions made by \citep{liang2018just}; however, a closer look at \citep{el2010spectrum} reveals that it is enough to assume differentiability in a neighborhood of $0$. Hence, the upper bounds of \citep{liang2018just} can be extended to the case of Laplace kernel, under the high-dimensional scaling $d\asymp n$.

The ``width'' parameter $c$ in \eqref{def:laplace_kernel} plays an important role. In particular, the upper bounds of \citep{liang2018just} were only shown in the specific regime of this parameter, $c\asymp\sqrt{d}$. The choice of $c$ presents a key difficulty for proving a lower bound: perhaps a clever data-dependent choice can yield a good estimator  even in low-dimensional situations? We prove a strong lower bound: \emph{no choice of $c$ can make the interpolation method \eqref{def:rkhs_interpolation} consistent if $d$ is a constant}.

The main theorem can be informally summarized as follows. If $Y_i$ are noisy observations of $f_0(X_i)$ at random points $X_i$, $i=1,\ldots,n$, the minimum-norm interpolant $\algo_c$ --- for the case of Laplacian kernel with any data-dependent choice of width $c$ --- is inconsistent, in the sense that with probability close to $1$,
$$\mathbb{E}_{X\sim \mathcal{P}}(\algo_{c}(X)-f_0(X))^2\ge \Omega_d(1).$$
Here $\mathcal{P}$ is the marginal distribution of $X$, $f_0$ is the regression function, and the order notation $\Omega_d$ stresses the fact that $d$ is a constant.

\section{Main Theorem}

Let $f_0$ be an unknown smooth function over $\Omega=\overline{B_{\mathbb{R}^d}(0,1)}$ that is not identically zero, and $\mathcal{P}$ an unknown distribution over $\Omega$ with probability density function $\rho$ bounded as
\begin{equation}
 	0<c_ \rho\le \rho\le C_ \rho.
 \end{equation}
 Suppose $X_1,\cdots,X_n$ are sampled i.i.d. according to $\mathcal{P}$, and 
 \begin{equation}
  	Y_i=f_0(X_i)+ \xi_i
  \end{equation}
with $\xi_i$ assumed to be i.i.d. noise with $\mathbb{P}(\xi_i=+ 1)=\mathbb{P}(\xi_i=- 1)=\frac{1}{2}$. 

We shall use $\mathcal{S}$ to denote the collection $\{(X_i,Y_i)\}_{i=1}^n$. Let $\algo_{c}$ be the minimum-norm function interpolating $(X_i,Y_i)$, with respect to Laplace kernel $K_c(x,y)= c^de^{-c\|x-y\|}$. 
\begin{mythm}
	For fixed $n$ and odd dimension $d$, with probability at least $1-O\left(\frac{1}{\sqrt{n}}\right)$ over the draw of $\mathcal{S}$,
 \begin{equation}
 	\forall c>0,~~~ \mathbb{E}_{X\sim \mathcal{P}}(\algo_{c}(X)-f_0(X))^2\ge \Omega_d(1).
 \end{equation}	
\end{mythm}

\begin{remark}
	We remark that the lower bound holds for any data-dependent choice $c$. The requirement that $d$ be odd is for technical simplicity, and we believe that our results can be extended to even dimensions by using more complicated tools in harmonic analysis. The assumption of binary noise process is for brevity; the noise magnitude can be changed by simple rescaling.
\end{remark}

For regularized least squares \eqref{def:rkhs}, the parameter $\lambda>0$ leads to a control of the norm of $\algo$. In the absence of explicit regularization, such a complexity control is more difficult to establish. Intuitively, the norm of the solution should be related to distances between datapoints, since the interpolating solution fits the noisy function values (separated by a constant), implying a large derivative if datapoints are close. More precisely, given the values $X_1,\ldots,X_n$, we define
\begin{equation}
	r_i:=\min(\min\limits_{j\neq i}{\|X_i-X_j\|},\text{dist}(X_i, \partial \Omega))
\end{equation}
for each $i=1,\ldots,n$. Analyzing the behavior of the random variables $r_i$ underlies the main proofs in this paper. While it is known that $\mathbb{E} [r_i] \lesssim n^{-1/2}$ \citep{gyorfi2006distribution}, our proofs require more delicate control of the tails of powers of these variables, including $r_i^{-1}$. As we show, the estimation error can be related to these random quantities, via Gagliardo-Nirenberg interpolation inequalities and control of higher-order derivatives.

\section{Proof}
\subsection{Outline}

\begin{enumerate}[(i)]
	\item We show that in odd dimension $d$, the RKHS norm has an explicit form, equal to a Sobolev norm.
	\item As the RKHS norm becomes the Sobolev norm, we can control ``smoothness'' of $\algo_{c}$ by controlling the RKHS norm. Since $\algo_{c}$ and $f_0$ differ on points $X_i$ by the amount $\xi_i$, and both functions are ``smooth'', we can choose small regions around $X_i$ such that the squared loss over these regions can be lower bounded. Unfortunately, the lower bound becomes vacuous as $c$ goes to infinity. Hence, we need a different strategy for ``large'' $c$. 
	\item When $c$ is large, the RKHS norm approximates the $L^2$-norm of $\mathbb{R}^d$. We then show that after $c$ passes a certain threshold, the $L^2$-norm of $\algo_{c}$ becomes smaller than a constant fraction of the norm of $f_0$, implying a lower bound on the total squared loss. 
	\item Remarkably, the two distinct lower bounds in $(ii)$ and $(iii)$ cover all the choices of $c$, a result that is not immediately evident.
\end{enumerate}




More specifically, we shall show that

\begin{prop}[First method]
	\label{prop:first}
	Fix a positive constant $A>0$. Then with probability at least $1-O_{d,\rho, A}\left(\frac{1}{\sqrt{n}}\right)$, for any $c\le A\sqrt[d]{n}$ we have
	\begin{equation}
		L(\algo_{c}) \deq \mathbb{E}\left(\algo_{c}(X)-f_0(X)\right)^2\ge \Omega_{d, \rho, f_0, A}\left(1\right).
	\end{equation}
\end{prop}

\begin{prop}[Second Method]
	\label{prop:second}
	There exists a constant $B=B(d, \rho, f_0)>0$ independent of $n$ such that with probability at least $1-O_{d, \rho}\left(\frac{1}{\sqrt{n}}\right)$, for any $c>B\sqrt[d]{n}$ we have
\begin{equation}
	\mathbb{E}\left(\algo_{c}(X)-f_0(X)\right)^2\ge \Omega_{d, \rho,f_0}\left(1\right).
\end{equation}
\end{prop}

Now we take the constant $A$ in the first method to be equal to $B $ and combine the two propositions, concluding that with high probability
\begin{equation}
	\forall c\in \mathbb{R},~~ L(\algo_{c})\ge \Omega(1).
\end{equation}

\subsection{Notation}

We work with the RKHS $\cH_c$ corresponding to the Laplace kernel \eqref{def:laplace_kernel}. The subscript emphasizes our focus on the width $c$. The inner product in $\cH_c$ is denoted by $\inner{f,g}_{\cH_c}$, and $\norm{f}^2_{\cH_c} = \inner{f}_{\cH_c}$ denotes the squared norm. Note that here we use the scaling as described in Proposition \ref{sunephone_is_always_my_bigbrother} so that

\begin{equation}
	\langle f\rangle_{\cH_c}=\sum\limits_{i=0}^{\frac{d+1}{2}}\binom{\frac{d+1}{2}}{i}c^{-2i}\langle f\rangle_i=\|f\|_{L^2(\mathbb{R}^d)}+\sum\limits_{i=1}^{\frac{d+1}{2}}\binom{\frac{d+1}{2}}{i}c^{-2i}\langle f\rangle_i
\end{equation}
where
\begin{equation}
	\langle f\rangle_i\deq\int_{\mathbb{R}^d}|\mathcal{F}f|^2\|p\|^{2i}dp=C_{d,i}\|D^if\|_{L^2(\mathbb{R}^d)}^2.
\end{equation}

\subsection{First Method: Control of H\"{o}lder Continuity}


\begin{proof}[Proof of Proposition~\ref{prop:first}]

	Denoting $f \deq \algo_{c}-f_0$, 
\begin{align}
		\mathbb{E}\left(\algo_{c}(X)-f_0(X)\right)^2 \ge \Omega_{d, \rho}\left(\|f\|_{L^2(\Omega)}^2\right).
\end{align}
Hence, we need only to give a lower bound to $\|f\|^2_{L^2(\Omega)}$. From Proposition \ref{Holder}, for any $I\subset [n]$, 
	\begin{align}
			&\|f\|_{L^2(\Omega)}^2\ge\min \left\{1, \Omega_d\left(\left(\frac{\min\limits_{i\in I}r_i^{-d-1}\sum\limits_{i\in I}r_i^{d}f(X_i)^2}{\max\limits_{i\in I}r_i^{-d-1}+c^{d+1} \langle f\rangle_{\cH_c}}\right)^d\sum_{i\in I}r_i^{d}f(X_i)^2\right)\right\}.
	\end{align}

	We can prove the proposition by giving upper bounds for $\max\limits_{i\in I}r_i^{-d-1}$ and $c^{d+1} \langle f\rangle_{\cH_c}$ and lower bounds for $\min\limits_{i\in I}r_i^{-d-1}$ and $\sum\limits_{i\in I}r_i^{d}f(X_i)^2$.

\begin{enumerate}[\textbf{\text{Estimate }}A.]
	\item

From Proposition \ref{ri2}, with probability $1-O_{d, \rho}(\frac{1}{\sqrt{n}})$ there is a subset $I\subset [n]$ of size at least $\frac{9}{10}n$ such that 
\begin{equation}
\label{eq:sbst}
	\Omega_{d, \rho}\left(n^{-\frac{1}{d}}\right)\le\min_{i\in I}r_i\le\max_{i\in I}r_i\le O_{d, \rho}\left(n^{-\frac{1}{d}}\right).
\end{equation}

Hence,
\begin{equation}
	\Omega_{d, \rho}\left(n^{\frac{d+1}{d}}\right)\le\min_{i\in I}r_i^{-d-1}\le\max_{i\in I}r_i^{-d-1}\le O_{d, \rho}\left(n^{ \frac{d+1}{d}}\right).
\end{equation}
\item

Note that for any $i$,
\begin{equation}
	f(X_i)^2=(\algo_c(X_i)-f_0(X_i)^2=(Y_i-f_0(X_i))^2=\xi_i^2=1,
\end{equation}

Then applying equation \eqref{eq:sbst} we get
\begin{equation}
	\sum\limits_{i\in I} r_i^d f(X_i)^2\ge \Omega_{d, \rho }\left(\sum_{i\in I} \left(n^{-\frac{1}{d}}\right)^d \cdot 1\right)\ge \Omega_{d, \rho, f_0, A}(1).
\end{equation}
\item

From Proposition \ref{interpolation1}, with probability $1-O_{d,\rho}\left(\frac{1}{\sqrt{n}}\right)$

\begin{equation}
	\begin{split}
		c^{d+1}\langle \algo_{c}\rangle_{\cH_c}&\le
		c^{d+1}\left(\frac{1}{3}\|f_0\|_{L^2(\omega)}^2+O_{d, \rho, f_0}\left(\frac{\sqrt[d]{n}}{c}\left(1+\frac{\sqrt[d]{n}}{c}\right)^{d}\right)\right)\\
		&\le
		O_{d, \rho, f_0}\left(c^{d+1}+ \sqrt[d]{n}\left(c+ \sqrt[d]{n} \right)^{d} \right)\\
		&\le
		O_{d, \rho, f_0}\left(A^{d+1}n^{\frac{d+1}{d}}+ \sqrt[d]{n}\left(A\sqrt[d]{n}+ \sqrt[d]{n} \right)^{d} \right)\\
		&=O_{d, \rho, f_0,A}\left(n^{\frac{d+1}{d}}\right).
	\end{split}
\end{equation}

It then follows that
\begin{equation}
	c^{d+1}\langle f\rangle_{\cH_c}\le 2 c^{d+1}\langle\algo_{c}\rangle_{\cH_c}+2c^{d+1}\langle f_0\rangle_{\cH_c}\le O_{d, \rho ,f_0, A}(n^{\frac{d+1}{d}}).
\end{equation}
 
\end{enumerate}

All the upper and lower bounds have been obtained. Then, with probability $1-O_{d,\rho}\left(\frac{1}{\sqrt{n}}\right)$
\begin{equation}
	\left(\frac{\min\limits_{i\in I}r_i^{-d-1}\sum_{i\in I}r_i^{d}f(X_i)^2}{\max\limits_{i\in I}r_i^{-d-1}+c^{d+1} \langle f\rangle_{\cH_c}}\right)^d\sum_{i\in I}r_i^{d}f(X_i)^2\ge\Omega_{d, \rho, f_0, A}(1).
\end{equation}

As a result, with probability at least $1-O_{d,\rho,f_0, A}\left(\frac{1}{\sqrt{n}}\right)$,
\begin{equation}
	\begin{split}
		L(\algo_{c}) =\mathbb{E}\left(\algo_{c}(X)-f_0(X)\right)^2\ge \Omega_{d, \rho, f_0, A}\left(1\right).
	\end{split}
\end{equation}
\end{proof}

\subsection{Second Method: Control of \texorpdfstring{$L^2$}{Lg} norm}
\begin{proof}[Proof of Proposition~\ref{prop:second}]

We need only to show the existence of $B$ such that
\begin{equation}
	\forall c>B\sqrt[d]{n},~~ \|\algo_{c} -f_0  \|^2_{L^2(\Omega)}\ge \Omega_{d, \rho,f_0}\left(1\right).
\end{equation}

From equation (\ref{interpolation_eq}) in Proposition \ref{interpolation1},

\begin{equation}
	\begin{split}
		\langle \algo_{c}\rangle_{\cH_c}&\le
		\frac{1}{3}\|f_0\|_{L^2(\Omega)}^2+O_{d, \rho, f_0}\left(\frac{\sqrt[d]{n}}{c}\left(1+\frac{\sqrt[d]{n}}{c}\right)^{d}\right)\\
		&\le
		\frac{1}{3}\|f_0\|_{L^2(\Omega)}^2+O_{d, \rho, f_0}\left(\frac{1}{B}\left(1+\frac{1}{B}\right)^{d}\right)\\
	\end{split}
\end{equation}

Then for $B=B(d, \rho, f_0)$ large enough,

\begin{equation}
	\begin{split}
		\langle \algo_{c}\rangle_{\cH_c}&\le
		\frac{1}{3}\|f_0\|_{L^2(\Omega)}^2+\frac{1}{3}\|f_0\|_{L^2(\Omega)}^2\\
		&\le
		\frac{2}{3}\|f_0\|_{L^2(\Omega)}^2.
	\end{split}
\end{equation}
Then
\begin{equation}
	\|\hat{f}-f_0\|_{L^2(\Omega)}\ge
	\|f_0\|_{L^2(\Omega)}-\|\algo_{c}\|_{L^2(\Omega)}
	\ge \left(1-\sqrt{\frac{2}{3}}\right) \|f_0\|_{L^2(\Omega)}=\Omega_{d, \rho,f_0}\left(1\right).
\end{equation}

\end{proof}
\appendix 

\section{Explicit Form of RKHS norm}

In this section, we provide an expression, up to constant factors, for the RKHS norm corresponding to the Laplace kernel, along with the associated eigenfunctions and eigenvalues.

\begin{prop}
\label{sunephone_is_always_my_bigbrother}
	Consider the kernel $K_c(x,y)=c^d e^{-c\|x-y\|}$ in $\mathbb{R}^d$ with $d$ odd. The corresponding RKHS norm is given by
	\begin{equation}
	\langle f\rangle_{\cH_c}\sim\int_{\mathbb{R}^d}|\mathcal{F}f|^2(1+\|p\|^2/c^2)^{\frac{d+1}{2}}dp\sim\sum\limits_{i=0}^{\frac{d+1}{2}}\binom{\frac{d+1}{2}}{i}c^{-2i}\langle f\rangle_i.
\end{equation}
where
\begin{equation}
	\langle f\rangle_i=\int_{\mathbb{R}^d}|\mathcal{F}f|^2\|p\|^{2i}dp=C_{d,i}\|D^if\|_{L^2(\mathbb{R}^d)}^2.
\end{equation}
and the Fourier transformation $\mathcal{F}$ is chosen such that
\begin{equation}
	\langle f\rangle_0=\|f\|_{L^2(\Omega)}^2.
\end{equation}

As scaling does not change the output of the algorithm, we take the convention that
\begin{equation}
	\langle f\rangle_{\cH_c}=\sum\limits_{i=0}^{\frac{d+1}{2}}\binom{\frac{d+1}{2}}{i}c^{-2i}\langle f\rangle_i=\|f\|_{L^2(\mathbb{R}^d)}+\sum\limits_{i=1}^{\frac{d+1}{2}}\binom{\frac{d+1}{2}}{i}c^{-2i}\langle f\rangle_i
\end{equation}

\end{prop}

\begin{proof}
Consider the integral operator
\begin{equation}
	T_Kf(x)=\int_y K(x,y)f(y)dy.
\end{equation}
We have
\begin{equation}
	\langle f,g\rangle_{\cH_c}=\langle f,T^{-1}_Kg\rangle_{L^2(\mathbb{R}^d)}.
\end{equation}
An eigenspace-decomposition of $T_K$ immediately gives the form of the inner product in the RKHS. Since $K_c(x,y)=k(x-y)$ with $k(x)=c^de^{-c\|x\|}$, it is easy to verify that the family $\{h_p(x)=e^{i p\cdot x}\}_{p\in \mathbb{R}^d}$ are eigenfunctions of $T_K$:
\begin{equation}
	T_Kh_p(x)=\int_y k(x-y)e^{ip\cdot y}dy=\lambda(p)h_p(x)
\end{equation}
where 
\begin{equation}
	\lambda(p)=\int_{y}k(x-y)e^{ip\cdot (y-x)}dy=\int_{x}k(x)e^{-ip\cdot x}dx.
\end{equation}
Therefore, the inner product of RKHS can be written as
\begin{equation}
	\langle f,g\rangle_{\cH_c}= \int_{x,p,y}\frac{1}{\lambda(p)^{-1}}f(x)^*h_p(x)h_p(y)^*g(y)dxdpdy
\end{equation}
which can be further rewritten as:
\begin{equation}
	\langle f,g\rangle_{\cH_c}=\int_{p}\frac{1}{\lambda(p)^{-1}}\mathcal{F}f(p)^*\mathcal{F}g(p)dp.
\end{equation}

Now for $\lambda(p)$, we have
\begin{equation}
	\lambda= \mathcal{F} k.
\end{equation}

In fact, $\lambda(p)$ can be explicitly computed (see e.g. \citep[Thm 1.4]{stein1971introduction}):
\begin{equation}
	\begin{split}
		\lambda(p)&=c^d\int_{\mathbb{R}^d} e^{-c\|x\|}e^{-ipx}dx\\
		&=\int_{\mathbb{R}^d} e^{-\|x\|}e^{-ipx/c}dx\\
		&=\int_{\mathbb{R}^d} \left(\frac{1}{\sqrt{\pi}}\int_0^\infty \frac{e^{- \eta}}{\sqrt{\eta}}e^{- \|x\|^2/4 \eta }d \eta\right)e^{-ipx/c}dx\\
		&=\frac{1}{\sqrt{\pi}}\int_0^\infty  \frac{e^{- \eta}}{\sqrt{\eta}}\left( \int_{\mathbb{R}^d} e^{- \|x\|^2/4 \eta }e^{-ipx/c}dx\right) d\eta\\
		&=\frac{1}{\sqrt{\pi}}\int_0^\infty  \frac{e^{- \eta}}{\sqrt{\eta}} (4 \pi \eta)^{d/2}e^{- \eta \|p\|^2/c^2} d\eta\\
		&=\frac{2^d \pi^{(d-1)/2}\Gamma(\frac{d+1}{2})}{(1+\|p\|^2/c^2)^{(d+1)/2}}.
	\end{split}
\end{equation}
Then
\begin{equation}
	\lambda(p)^{-1}=\frac{(1+\|p\|^2/c^2)^{(d+1)/2}}{2^d \pi^{(d-1)/2}\Gamma(\frac{d+1}{2})}=
	\frac{\sum\limits_{i=0}^{(d+1)/2}\binom{\frac{d+1}{2}}{i}\|p\|^{2i}/c^{2i}}{2^d \pi^{(d-1)/2}\Gamma(\frac{d+1}{2})}
\end{equation}
and
\begin{align}
		\int_{p}\frac{1}{\lambda(p)^{-1}}\mathcal{F}f(p)^*\mathcal{F}g(p)dp
		&=\int_{p}\sum\limits_{i=0}^{(d+1)/2}\frac{\binom{\frac{d+1}{2}}{i}\|p\|^{2i}/c^{2i}}{2^d \pi^{(d-1)/2}\Gamma(\frac{d+1}{2})}\mathcal{F}f(p)^*\mathcal{F}g(p)dp\\
		&=\sum\limits_{i=0}^{(d+1)/2}\frac{\binom{\frac{d+1}{2}}{i}/c^{2i}}{2^d \pi^{(d-1)/2}\Gamma(\frac{d+1}{2})}\int_{p}\|p\|^{2i}\mathcal{F}f(p)^*\mathcal{F}g(p)dp,
\end{align}
implying the result.

\end{proof}

\section{Bounds of Average Separation}

\subsection{Main Claims}


\begin{prop}
\label{ri1}
	There are constants $C_1,C_2$ depending on $d$, such that with probability $1- O(\frac{1}{\sqrt{n}})$, the following holds for all $-1\le k\le d$:
\begin{equation}
	C_1 n^{-\frac{k}{d}}\le \frac{1}{n}\sum\limits_{i=1}^n r_i^k\le C_2 n^{-\frac{k}{d}}.
\end{equation}
\end{prop}

Now, since we have the following inequality 
\begin{equation}
	\left(\frac{1}{n}\sum\limits_{i=1}^n r_i^{-1}\right)^{-k}\le\frac{1}{n}\sum\limits_{i=1}^n r_i^k\le \left(\frac{1}{n}\sum\limits_{i=1}^n r_i^d\right)^{\frac{k}{d}}
\end{equation}
for all $-1\le k\le d$, we need only to prove that with high probability
\begin{equation}
	\frac{1}{n}\sum\limits_{i=1}^n r_i^d\lesssim n^{-1}
\end{equation}
and
\begin{equation}
	\frac{1}{n}\sum\limits_{i=1}^n r_i^{-1}\lesssim n^{\frac{1}{d}}.
\end{equation}

\begin{prop}
\label{ri2}
	For any $0< \alpha< 1$, there is constant $C_1',C_2'$ depending on $\alpha, d$, such that  with probability $1- O(\frac{1}{\sqrt{n}})$, we have
	\begin{equation}
		|\{i:C_1'/\sqrt[d]{n}\le r_i\le C_2'/\sqrt[d]{n}\}|\ge \alpha n
	\end{equation}
\end{prop}

\begin{proof}
	With probability at least $1-O(\frac{1}{\sqrt{n}})$ for all $-1\le k\le d$, for constants $C_1,C_2$,
	\begin{equation}
		C_1 n^{-\frac{k}{d}}\le \frac{1}{n}\sum\limits_{i=1}^n r_i^k\le C_2 n^{-\frac{k}{d}}
	\end{equation}

Let $\beta=1- \frac{1}{2}(1- \alpha)=\frac{1+ \alpha}{2}$.

Let $I_1$ be a subset of $[n]$ of size $\text{ceil}(\beta n)$ such that $\forall i\in I_1,j\in [n]\setminus I_1, r_i\ge r_j$. Let $r=\min_{i\in I_1}r_i$. Then
\begin{equation}
	C_2\sqrt[d]{n}\ge \frac{1}{n}\sum_i r_i^{-1}\ge \frac{1}{n}\sum_{i\in I_1}r_i^{-1}\ge \frac{1}{n}\frac{\beta n}{r}= \frac{\beta}{r}.
\end{equation}

Then $r\ge C_2/(\beta\sqrt[d]{n})$. Then take $C_1'=C_2/ \beta$. So for any $i\in I_1, r_i\ge C_1'/\sqrt[d]{n}$. Similarly, there is a subset $I_2$ of $[n]$ of size $\text{ceil}(\beta n)$ such that $\forall i\in I_2, r_i\ge C_2'/\sqrt[d]{n}$. Note that $|I_1\cap I_2|\ge \alpha n$, concluding the proof.

\end{proof}
\subsection{Average of \texorpdfstring{$r_i^{d}$}{Lg}}

The following is always true:
\begin{equation}
	\sum\limits_{i=1}^n r_i^d\lesssim \sum\limits_{i=1}^n m(B(X_i,\frac{1}{2}r_i))\le m(\Omega)\lesssim 1
\end{equation}
Then the result follows.

\subsection{Average of \texorpdfstring{$r_i^{-1}$}{Lg}}
\subsubsection{Strategy}
We shall use Chebyshev's inequality to bound average of $r_i^{-1}$, and thus we need to estimate $\text{Cov}(r_i^{-1},r_j^{-1})$. This step is not direct because $r_i,r_j$ are not independent: both depend on $X_i$ and $X_j$.

We define $\tilde r_i,\tilde r_j$ for any fixed pair of $(i,j)$ such that
\begin{itemize}
	\item $\tilde r_i=r_i,\tilde r_j=r_j$ with high probability
	\item $\tilde r_i$ is independent w.r.t $X_j$, $\tilde r_j$ is independent w.r.t $X_i$
\end{itemize}

We will then show that $Cov(\tilde r_i,\tilde r_j)$ is small and that the difference between $Cov(r_i,r_j)$ and $Cov(\tilde r_i,\tilde r_j)$ is small. Applying Chebyshev's inequality then yields the result.
\subsubsection{Upper bound for \texorpdfstring{$\mathbb{E}[r_i^{-1}]$}{Lg} and \texorpdfstring{$\mathbb{E}[r_i^{-2}]$}{Lg}}
\begin{equation}
	\begin{split}
		\mathbb{P}(r_i<r)=1- m_\mathcal{P}(B(X_i,r)^c)^n\le n m_\mathcal{P}(B(X_i,r))\lesssim nr^d.
	\end{split}
\end{equation}

Then

\begin{equation}
	\begin{split}
		\mathbb{E}r_i^{-1}&=\mathbb{E}\int_0^\infty \mathbb{I}(r_i^{-1}>s)ds\\
		&=\int_0^\infty \mathbb{E}\mathbb{I}(r_i^{-1}>s)ds\\
		&=\int_0^\infty \mathbb{P}(r_i^{-1}>s)ds\\
		&\le\int_0^\infty \min(1, C_dns^{-d})ds\\
		&= s_0+ C_dns_0^{1-d}/(d-1)\text{ where }C_dns_0^{-d}=1\\
		&=\frac{d}{d-1}s_0\\
		&=\frac{d}{d-1}\sqrt[d]{C_dn}
	\end{split}
\end{equation}

and

\begin{equation}
	\begin{split}
		\frac{1}{2}\mathbb{E}r_i^{-2}&=\mathbb{E}\int_0^\infty s\mathbb{I}(r_i^{-1}>s)ds\\
		&=\int_0^\infty s\mathbb{E}\mathbb{I}(r_i^{-1}>s)ds\\
		&=\int_0^\infty s\mathbb{P}(r_i^{-1}>s)ds\\
		&\le\int_0^\infty s\min(1, C_dns^{-d})ds\\
		&= \frac{1}{2}s_0^2+ C_dns_0^{2-d}/(d-2)\text{ where }C_dns_0^{-d}=1\\
		&=\left(\frac{1}{2}+\frac{1}{d-2}\right)s_0^2\\
		&=\left(\frac{1}{2}+\frac{1}{d-2}\right)(\sqrt[d]{C_dn})^2.
	\end{split}
\end{equation}

Hence,
\begin{equation}
	\mathbb{E}r_i^{-2}\le \frac{d}{d-2}(C_dn)^{\frac{2}{d}}
\end{equation}
\subsubsection{Estimate of \texorpdfstring{$\text{Cov}(\frac{1}{\tilde r_i},\frac{1}{\tilde r_j})$}{Lg}}

Define 

\begin{equation}
	\tilde r_i:=\min(\min\limits_{k\neq i,j}|X_k-X_i|,\text{dist}(\partial \Omega,X_i))
\end{equation}
and
\begin{equation}
	\tilde r_j:=\min(\min\limits_{k\neq i,j}|X_k-X_j|,\text{dist}(\partial \Omega,X_j))
\end{equation}

Then
\begin{equation}
	r_i=\min(\tilde r_i,|X_i-X_j|),r_j=\min(\tilde r_j,|X_i-X_j|)
\end{equation}

and $\tilde r_j$ is independent of $X_i$ and $\tilde r_i$ is independent of $X_j$.

\begin{align*}
	&\mathbb{E}[\frac{1}{\tilde r_i\tilde r_j}]-\mathbb{E}[\frac{1}{\tilde r_i}]\mathbb{E}[\frac{1}{\tilde r_j}]\\
	&=\mathbb{E}_{X_i,X_j}[\mathbb{E}[\frac{1}{\tilde r_i\tilde r_j}|X_i,X_j]]-\mathbb{E}_{X_i}[\mathbb{E}[\frac{1}{\tilde r_i}|X_i]]\mathbb{E}_{X_j}[\mathbb{E}[\frac{1}{\tilde r_j}|X_j]]\\
	&=\mathbb{E}_{X_i,X_j}[\mathbb{E}[\frac{1}{\tilde r_i\tilde r_j}|X_i,X_j]]-\mathbb{E}_{X_i}\left[\mathbb{E}\left[\frac{1}{\tilde r_i}|X_i\right]\mathbb{E}_{X_j}\left[\mathbb{E}[\frac{1}{\tilde r_j}|X_j]\right]\right]\\
	&=\mathbb{E}_{X_i,X_j}\left[\mathbb{E}[\frac{1}{\tilde r_i\tilde r_j}|X_i,X_j]\right]-\mathbb{E}_{X_i,X_j}\left[\mathbb{E}\left[\frac{1}{\tilde r_i}|X_i\right]\mathbb{E}\left[\frac{1}{\tilde r_j}|X_j\right]\right]\text{ (indep. between $X_i$ and $X_j$)}\\
		&=\mathbb{E}_{X_i,X_j}\left[\mathbb{E}\left[\frac{1}{\tilde r_i \tilde r_j}\Big|X_i,X_j\right]-\mathbb{E}\left[\frac{1}{\tilde r_i}\Big| X_i,X_j\right]\mathbb{E}\left[\frac{1}{\tilde r_j}\Big| X_i,X_j\right]\right]
\end{align*}
where we used independence between $\tilde r_i$ and $X_j$ and between $\tilde r_j$ and $X_i$. The last expression can be written as
\begin{align*}
		&=\mathbb{E}_{X_i,X_j}\Bigg[\mathbb{E}\left[\int_0^\infty ds \mathbb{I}(\tilde r_i^{-1}>s)\int_0^\infty dt \mathbb{I}(\tilde r_j^{-1}>t)\Big| X_i,X_j\right]\\
		&-\left(\mathbb{E}\int_0^\infty ds \mathbb{I}(\tilde r_i^{-1}>s)\Big| X_i,X_j\right)\left(\mathbb{E}\int_0^\infty dt \mathbb{I}(\tilde r_j^{-1}>t)\Big| X_i,X_j\right)\Bigg]\\
		&=\mathbb{E}_{X_i,X_j}\Bigg[\int_0^\infty ds\int_0^\infty dt\mathbb{E}\left[\mathbb{I}(\tilde r_i^{-1}>s,\tilde r_j^{-1}>t)\Big| X_i,X_j\right]\\
		&-\int_0^\infty \mathbb{E}\left[\mathbb{I}(\tilde r_i^{-1}>s)\Big| X_i,X_j\right]ds\int_0^\infty\mathbb{E}\left[\mathbb{I}(\tilde r_j^{-1}>t)\Big| X_i,X_j\right] dt \Bigg]\\
		&=\mathbb{E}_{X_i,X_j}\left[\int_0^\infty\int_0^\infty \Big(\mathbb{P}\left[\tilde r_i^{-1}>s,\tilde r_j^{-1}>t\Big|X_i,X_j\right]-
		\mathbb{P}\left[\tilde r_i^{-1}>s\Big|X_i,X_j\right]  \mathbb{P}\left[\tilde r_j^{-1}>t\Big| X_i,X_j\right]  \Big)dsdt\right].
\end{align*}

Now,
\begin{equation}
	\begin{split}
		&\mathbb{P}[r_i^{-1}>s,r_j^{-1}>t|X_i,X_j]]\\
&=1-\mathbb{P}[r_i^{-1}<s|X_i,X_j]]-\mathbb{P}[r_j^{-1}<t|X_i,X_j]]+\mathbb{P}[r_i^{-1}<s,r_j^{-1}<t|X_i,X_j]\\
&\mathbb{P}[r_i^{-1}>s|X_i,X_j]\mathbb{P}[r_j^{-1}>t|X_i,X_j]\\
		&=1-\mathbb{P}[r_i^{-1}<s|X_i,X_j]-\mathbb{P}[r_j^{-1}<t|X_i,X_j]+\mathbb{P}[r_i^{-1}<s|X_i,X_j]\mathbb{P}[r_j^{-1}<t|X_i,X_j]
	\end{split}
\end{equation}

Then

\begin{equation}
	\begin{split}
	&\mathbb{E}[\frac{1}{\tilde r_i\tilde r_j}]-\mathbb{E}[\frac{1}{\tilde r_i}]\mathbb{E}[\frac{1}{\tilde r_j}]\\
	&=\mathbb{E}_{X_i,X_j}\left[\int_0^\infty\int_0^\infty \Big(\mathbb{P}\left[\tilde r_i^{-1}>s,\tilde r_j^{-1}>t\Big|X_i,X_j\right]-
	\mathbb{P}\left[\tilde r_i^{-1}>s\Big|X_i,X_j\right]  \mathbb{P}\left[\tilde r_j^{-1}>t\Big| X_i,X_j\right]  \Big)dsdt\right]\\
	&=\mathbb{E}_{X_i,X_j}\left[\int_0^\infty\int_0^\infty \Big(
	\mathbb{P}\left[\tilde r_i^{-1}<s\Big|X_i,X_j\right]  \mathbb{P}\left[\tilde r_j^{-1}<t\Big| X_i,X_j\right]  -
	\mathbb{P}\left[\tilde r_i^{-1}<s,\tilde r_j^{-1}<t\Big|X_i,X_j\right]\Big)dsdt\right]\\
	&=\mathbb{E}_{X_i,X_j}\left[\int_0^\infty\int_0^\infty \Big(m_{\mathcal{P}}(B(X_i,s^{-1})^c)^{n-2}m_{\mathcal{P}}(B(X_j,t^{-1})^c)^{n-2} -
	m_{\mathcal{P}}((B(X_i,s^{-1})\cup B(X_j,t^{-1}))^c)^{n-2}\Big)dsdt\right]\\
	&=\mathbb{E}_{X_i,X_j}\left[\int_{R_0^{-1}}^\infty\int_{R_0^{-1}}^\infty \Big(m_{\mathcal{P}}(B(X_i,s^{-1})^c)^{n-2}m_{\mathcal{P}}(B(X_j,t^{-1})^c)^{n-2} -
	m_{\mathcal{P}}((B(X_i,s^{-1})\cup B(X_j,t^{-1}))^c)^{n-2}\Big)dsdt\right]\\
	\end{split}
\end{equation}
where $R_0= \text{diam}(\Omega)$ is a constant depending only on $d$.

When $s^{-1}+t^{-1}< |X_i-X_j|$, we have
\begin{equation}
	B(X_i,s^{-1})\cup B(X_j,t^{-1})=B(X_i,s^{-1})\sqcup B(X_j,t^{-1})
\end{equation}
where $\sqcup $ means disjoint union. Then

\begin{equation}
	\begin{split}
	&m_{\mathcal{P}}(B(X_i,s^{-1})^c)m_{\mathcal{P}}(B(X_j,t^{-1})^c) -
	m_{\mathcal{P}}((B(X_i,s^{-1})\cup B(X_j,t^{-1}))^c)\\
	&=m_{\mathcal{P}}(B(X_i,s^{-1})^c)m_{\mathcal{P}}(B(X_j,t^{-1})^c) -
	m_{\mathcal{P}}((B(X_i,s^{-1})\sqcup B(X_j,t^{-1}))^c)\\
	&=(1-m_\mathcal{P}(B(X_i,s^{-1})))(1-m_\mathcal{P}(B(X_j,t^{-1})))-(1-m_\mathcal{P}(B(X_i,s^{-1}))-m_\mathcal{P}(B(X_j,t^{-1})))\\
	&=m_\mathcal{P}(B(X_i,s^{-1}))m_\mathcal{P}(B(X_j,t^{-1}))\\
	&\ge 0
	\end{split}
\end{equation}

Since for $0\le x\le y\le 1$, $x^{n-2}-y^{n-2}\le (n-2)x^{n-3}(x-y)$, we have

\begin{equation}
	\begin{split}
		&0\le m_{\mathcal{P}}(B(X_i,s^{-1})^c)^{n-2}m_{\mathcal{P}}(B(X_j,t^{-1})^c)^{n-2} -
	m_{\mathcal{P}}((B(X_i,s^{-1})\cup B(X_j,t^{-1}))^c)^{n-2}\\
	&\le (n-3) m_{\mathcal{P}}(B(X_i,s^{-1})^c)^{n-3}m_{\mathcal{P}}(B(X_j,t^{-1})^c)^{n-3}m_\mathcal{P}(B(X_i,s^{-1}))m_\mathcal{P}(B(X_j,t^{-1}))\\\
	&\le (n-3)\left(\max(0,1- \frac{C_d}{s^{d}})\right)^{n-3}\left(\max(0,1- \frac{C_d}{t^{d}})\right)^{n-3}\frac{C_d'}{s^d}\frac{C_d'}{t^d}
	\end{split}
\end{equation}
where $C_d,C_d'$ are constants such that for any $B(x,r)\subset \Omega$
\begin{equation}
	C_dr^d\le m_\mathcal{P}(B(x,r))
	\le C'_dr^d .
\end{equation}

When $s^{-1}> \frac{\|X_i-X_j\|}{2}$, we have

\begin{equation}
	\begin{split}
		&m_{\mathcal{P}}(B(X_i,s^{-1})^c)^{n-2}m_{\mathcal{P}}(B(X_j,t^{-1})^c)^{n-2} -
	m_{\mathcal{P}}((B(X_i,s^{-1})\cup B(X_j,t^{-1}))^c)^{n-2}\\
	&\ge m_{\mathcal{P}}(B(X_i,s^{-1})^c)^{n-2}m_{\mathcal{P}}(B(X_j,t^{-1})^c)^{n-2} -
	m_{\mathcal{P}}((B(X_i,s^{-1}))^c)^{n-2}\\
	&\ge m_{\mathcal{P}}(B(X_i,s^{-1})^c)^{n-2}\cdot\min\{1,(n-2) m_\mathcal{P}(B(X_j,t^{-1}))\}\\
	&\ge -(\max(0,1-C_d s^{-d}))^{n-2}\min(1,(n-2)C_d' t^{-d})
	\end{split}
\end{equation}

and

\begin{equation}
	\begin{split}
		&m_{\mathcal{P}}(B(X_i,s^{-1})^c)^{n-2}m_{\mathcal{P}}(B(X_j,t^{-1})^c)^{n-2} -
	m_{\mathcal{P}}((B(X_i,s^{-1})\cup B(X_j,t^{-1}))^c)^{n-2}\\
	&\le m_{\mathcal{P}}(B(X_i,s^{-1})^c)^{n-2}m_{\mathcal{P}}(B(X_j,t^{-1})^c)^{n-2} -
	(1-m_{\mathcal{P}}(B(X_i,s^{-1}))-m_\mathcal{P}(B(X_j,t^{-1})))^{n-2}\\
	&\le (n-2)m_{\mathcal{P}}(B(X_i,s^{-1})^c)^{n-3}m_{\mathcal{P}}(B(X_j,t^{-1})^c)^{n-3}m_{\mathcal{P}}(B(X_i,s^{-1}))m_{\mathcal{P}}(B(X_j,t^{-1}))\\
	&\le (n-3)\left(\max(0,1- \frac{C_d}{s^{d}})\right)^{n-3}\left(\max(0,1- \frac{C_d}{t^{d}})\right)^{n-3}\frac{C_d'}{s^d}\frac{C_d'}{t^d}
	\end{split}
\end{equation}

Then

\begin{equation}
	\begin{split}
		&-(\max(0,1-C_d s^{-d}))^{n-2}\min(1,(n-2)C_d' t^{-d})\\
		&\le m_{\mathcal{P}}(B(X_i,s^{-1})^c)^{n-2}m_{\mathcal{P}}(B(X_j,t^{-1})^c)^{n-2} -
	m_{\mathcal{P}}((B(X_i,s^{-1})\cup B(X_j,t^{-1}))^c)^{n-2}\\
	&\le (n-3)\left(\max(0,1- \frac{C_d}{s^{d}})\right)^{n-3}\left(\max(0,1- \frac{C_d}{t^{d}})\right)^{n-3}\frac{C_d'}{s^d}\frac{C_d'}{t^d}
	\end{split}
\end{equation}

Similarly for $t^{-1}>  \frac{\|X_i-X_j\|}{2}$, we have

\begin{equation}
	\begin{split}
		&-(\max(0,1-C_d t^{-d}))^{n-2}\min(1,(n-2)C_d' s^{-d})\\
		&\le m_{\mathcal{P}}(B(X_i,s^{-1})^c)^{n-2}m_{\mathcal{P}}(B(X_j,t^{-1})^c)^{n-2} -
	m_{\mathcal{P}}((B(X_i,s^{-1})\cup B(X_j,t^{-1}))^c)^{n-2}\\
	&\le (n-3)\left(\max(0,1- \frac{C_d}{s^{d}})\right)^{n-3}\left(\max(0,1- \frac{C_d}{t^{d}})\right)^{n-3}\frac{C_d'}{s^d}\frac{C_d'}{t^d}
	\end{split}
\end{equation}

The upper bound are the same in all three cases, but the lower bounds are different.

\paragraph{
Upper bound for \texorpdfstring{$\text{Cov}(\tilde r_i^{-1},\tilde r_j^{-1})$}{s}}

We now put the above calculations together and estimate
\begin{equation}
	\begin{split}
	&\text{Cov}(\tilde r_i^{-1},\tilde r_j^{-1})\\
	&=\mathbb{E}[\frac{1}{\tilde r_i\tilde r_j}]-\mathbb{E}[\frac{1}{\tilde r_i}]\mathbb{E}[\frac{1}{\tilde r_j}]\\
	&=\mathbb{E}_{X_i,X_j}\left[\int_0^\infty\int_0^\infty \Big(m_{\mathcal{P}}(B(X_i,s^{-1})^c)^{n-2}m_{\mathcal{P}}(B(X_j,t^{-1})^c)^{n-2} -
	m_{\mathcal{P}}((B(X_i,s^{-1})\cup B(X_j,t^{-1}))^c)^{n-2}\Big)dsdt\right]\\
	&=\mathbb{E}_{X_i,X_j}
	\left[
	\int_{R_0^{-1}}^\infty\int_{R_0^{-1}}^\infty 
	\Big(m_{\mathcal{P}}(B(X_i,s^{-1})^c)^{n-2}m_{\mathcal{P}}(B(X_j,t^{-1})^c)^{n-2} -
	m_{\mathcal{P}}((B(X_i,s^{-1})\cup B(X_j,t^{-1}))^c)^{n-2}\Big)dsdt\right]\\
	&\le \mathbb{E}_{X_i,X_j}
	\int_{R_0^{-1}}^\infty\int_{R_0^{-1}}^\infty 
			 (n-3)\left(\max(0,1- \frac{C_d}{s^{d}})\right)^{n-3}\left(\max(0,1- \frac{C_d}{t^{d}})\right)^{n-3}\frac{C_d'}{s^d}\frac{C_d'}{t^d}
			 dsdt\\\
	&\le\mathbb{E}_{X_i,X_j}
	\frac{R_0^2}{4}
			\int_{R_0^{-1}}^\infty\int_{R_0^{-1}}^\infty 
			 (n-3)\left(\max(0,1- \frac{C_d}{s^{d}})\right)^{n-3}\left(\max(0,1- \frac{C_d}{t^{d}})\right)^{n-3}\frac{C_d'}{s^{d-1}}\frac{C_d'}{t^{d-1}}
			 dsdt\\
			 &\le\mathbb{E}_{X_i,X_j}\frac{R_0^2}{4}
			 \int_{0}^\infty\int_{0}^\infty
			 (n-3)\left(\max(0,1- \frac{C_d}{s^{d}})\right)^{n-3}\left(\max(0,1- \frac{C_d}{t^{d}})\right)^{n-3}\frac{C_d'}{s^{d-1}}\frac{C_d'}{t^{d-1}}
			 dsdt\\
			 &=\frac{R_0^2}{4}\frac{n-3}{d^2(n-2)^2}(C'_d/C_d)^2\\
			 &=O\left(\frac{1}{n}\right)
	\end{split}
\end{equation}
\paragraph{
Lower bound for \texorpdfstring{$\text{Cov}(\tilde r_i^{-1},\tilde r_j^{-1})$}{s}}

\begin{equation}
	\begin{split}
	&\mathbb{E}[\frac{1}{\tilde r_i\tilde r_j}]-\mathbb{E}[\frac{1}{\tilde r_i}]\mathbb{E}[\frac{1}{\tilde r_j}]\\
	&=\mathbb{E}_{X_i,X_j}\left[\int_0^\infty\int_0^\infty \Big(m_{\mathcal{P}}(B(X_i,s^{-1})^c)^{n-2}m_{\mathcal{P}}(B(X_j,t^{-1})^c)^{n-2} -
	m_{\mathcal{P}}((B(X_i,s^{-1})\cup B(X_j,t^{-1}))^c)^{n-2}\Big)dsdt\right]\\
	&=\mathbb{E}_{X_i,X_j}\left[\underbrace{
	\int_{\frac{2}{\|X_i-X_j\|}}^\infty\int_{\frac{2}{\|X_i-X_j\|}}^\infty \cdots dsdt}_A 
	+   \underbrace{\int_{0}^\infty\int^{\frac{2}{\|X_i-X_j\|}}_0 \cdots dsdt}_B
	+ \underbrace{\int^{\frac{2}{\|X_i-X_j\|}}_0\int_{0}^\infty\cdots dsdt}_C\right]\\
	\end{split}
\end{equation}

\begin{enumerate}[(a)]
	\item lower bound of $A$.

	\begin{equation}
	A\ge 0	
	\end{equation}
	\item lower bound of $B$.

	\begin{equation}
		\begin{split}
			B&\ge -\int^{\frac{2}{\|X_i-X_j\|}}_0\int_{0}^\infty
			(\max(0,1-C_d s^{-d}))^{n-2}\min(1,(n-2)C_d' t^{-d})
			dtds\\
			&\ge-\frac{2}{\|X_i-X_j\|}
			\left(\max\left\{0,1-C_d \left(\frac{\|X_i-X_j\|}{2}\right)^{d}\right\}\right)^{n-2}
			\int_{0}^\infty\min(1,(n-2)C_d t^{-d})dt\\
			&\ge-\frac{2}{\|X_i-X_j\|}
			\left(\max\left\{0,1-C_d \left(\frac{\|X_i-X_j\|}{2}\right)^{d}\right\}\right)^{n-2}
			\frac{1}{d-1}((n-2)C_d)^{\frac{1}{d}}.
		\end{split}
	\end{equation}

	Note that
	
	\begin{equation}
		\begin{split}
			&\mathbb{E}\left[
			\frac{2}{\|X_i-X_j\|}
			\left(\max\left\{0,1-C_d \left(\frac{\|X_i-X_j\|}{2}\right)^{d}\right\}\right)^{n-2}
			\Bigg|X_i\right]\\
			&= \int_0^{R_0} \frac{2}{r}
			\left(\max\left\{0,1-C_d \left(\frac{r}{2}\right)^{d}\right\}\right)^{n-2}
			d \mu_{\|X_i-X_j\||X_i}(r)\\
			&\lesssim
			\int_0^{R_0} \frac{2}{r}
			\left(\max\left\{0,1-C_d \left(\frac{r}{2}\right)^{d}\right\}\right)^{n-2}
			r^{d-1}dr\\
			&\lesssim
			\int_0^{R_0} 
			\left(\max\left\{0,1- \frac{R_0}{2^d} r^{d-1}\right\}\right)^{n-2}
			r^{d-2}dr\\
			&=
			\int_0^{R_0^{d-1}} 
			\left(\max\left\{0,1- \frac{R_0}{2^d} r^{d-1}\right\}\right)^{n-2}
			\frac{1}{d-1}d(r^{d-1})\\
			&\lesssim \frac{1}{n}
		\end{split}
	\end{equation}

	As a result,
	\begin{equation}
		\mathbb{E}_{X_i,X_j}B\gtrsim \frac{1}{n}
	\end{equation}
	\item Similarly for $C$, we have
	\begin{equation}
		\mathbb{E}_{X_i,X_j}C\gtrsim \frac{1}{n}
	\end{equation}
\end{enumerate}

Combining all the above inequalities, we have
\begin{equation}
	\mathbb{E}[\frac{1}{\tilde r_i\tilde r_j}]-\mathbb{E}[\frac{1}{\tilde r_i}]\mathbb{E}[\frac{1}{\tilde r_j}]\gtrsim \frac{1}{n}.
\end{equation}

\paragraph{Upper bound for \texorpdfstring{$|\text{Cov}(\frac{1}{\tilde r_i},\frac{1}{r_j})|$}{s}}
\begin{equation}
	|\text{Cov}(\frac{1}{\tilde r_i},\frac{1}{r_j})|=|\mathbb{E}[\frac{1}{\tilde r_i\tilde r_j}]-\mathbb{E}[\frac{1}{\tilde r_i}]\mathbb{E}[\frac{1}{\tilde r_j}]|
	\lesssim \frac{1}{n}.
\end{equation}
\subsubsection{Estimate for the difference between \texorpdfstring{$\text{Cov}(\frac{1}{\tilde r_i},\frac{1}{r_j})$}{s} and \texorpdfstring{$\text{Cov}(\frac{1}{\tilde r_i},\frac{1}{r_j})$}{s}}
\paragraph{Upper bound for \texorpdfstring{$\mathbb{E}|\tilde r_i^{-1}- r_i^{-1} |^2$}{s}}

We have
\begin{equation}
	|\tilde r_i^{-1}- r_i^{-1} |\le \frac{1}{\|X_i-X_j\|}\mathbb{I}\{\|X_i-X_j\|<\tilde r_i\}.
\end{equation}

Conditioned on $X_i$, $\|X_i-X_j\|$ and $\tilde r_i$ are, in fact, independent. Then
\begin{equation}
	\begin{split}
		\mathbb{E}[|\tilde r_i^{-1}- r_i^{-1} |^2|X_i,\tilde r_i]&\le \mathbb{E}[\frac{1}{\|X_i-X_j\|^2}\mathbb{I}\{\|X_i-X_j\|<\tilde r_i\}|X_i,\tilde r_i]\\
		&\lesssim \mathbb{E}[\tilde r_i^{d-2}|X_i,r_i]
	\end{split}
\end{equation}

Hence,
\begin{equation}
	\begin{split}
		\mathbb{E}[|\tilde r_i^{-1}- r_i^{-1} |^2 \le \mathbb{E}\tilde r_i^{d-2} \le \mathbb{E}[\tilde r_i^d]^\frac{d-2}{d} \lesssim n^{-\frac{d-2}{d}}
	\end{split}
\end{equation}

\subsubsection{Upper bound for \texorpdfstring{$\mathbb{E}|\tilde r_i^{-1}\tilde r_j^{-1}- r_i^{-1} r_j^{-1}|$}{s}}

\begin{equation}
	\begin{split}
		\mathbb{E}|\tilde r_i^{-1}\tilde r_j^{-1}- r_i^{-1} r_j^{-1}|&\le\mathbb{E}|\tilde r_i^{-1}\tilde r_j^{-1}- \tilde r_i^{-1} r_j^{-1}|+\mathbb{E}|\tilde r_i^{-1} r_j^{-1}- r_i^{-1} r_j^{-1}|\\
		&\le\sqrt{\mathbb{E}[\tilde r_i^{-2}]}\sqrt{\mathbb{E}[|\tilde r_j^{-1}- r_j^{-1}|^2]}+
		\sqrt{\mathbb{E}[r_j^{-2}]}\sqrt{\mathbb{E}[|\tilde r_i^{-1}- r_i^{-1}|^2]}\\
		&\le\sqrt{\mathbb{E}[r_i^{-2}]}\sqrt{\mathbb{E}[|\tilde r_j^{-1}- r_j^{-1}|^2]}+
		\sqrt{\mathbb{E}[r_j^{-2}]}\sqrt{\mathbb{E}[|\tilde r_i^{-1}- r_i^{-1}|^2]}\\
		&\lesssim \sqrt{n^{2/d}}\sqrt{n^{- \frac{d-2}{d}}}\\
		&\le n^{-\frac{d-4}{2d}}
	\end{split}
\end{equation}

\subsubsection{Upper bound for \texorpdfstring{$|\mathbb{E}[\tilde r_i^{-1}]\mathbb{E}[\tilde r_j^{-1}]-\mathbb{E}[r_i^{-1}]\mathbb{E}[r_j^{-1}]|$}{s}}

First,

\begin{equation}
	\|\mathbb{E}[\tilde r_i^{-1}]-\mathbb{E}[r_i^{-1}]\|\le\sqrt{\mathbb{E}[(\tilde r_i^{-1}-r_i^{-1})^2]}\lesssim n^{-\frac{d-2}{2d}}
\end{equation}
and
\begin{equation}
	\mathbb{E}\tilde r_i^{-1}\le \mathbb{E}r_i^{-1}\lesssim n^{\frac{1}{d}}.
\end{equation}

Then

\begin{equation}
	\begin{split}
		&|\mathbb{E}[\tilde r_i^{-1}]\mathbb{E}[\tilde r_j^{-1}]-\mathbb{E}[r_i^{-1}]\mathbb{E}[r_j^{-1}]|\\
		&=|\mathbb{E}[\tilde r_i^{-1}]\mathbb{E}[\tilde r_j^{-1}]-\mathbb{E}[\tilde r_i^{-1}]\mathbb{E}[r_j^{-1}]|+|\mathbb{E}[\tilde r_i^{-1}]\mathbb{E}[r_j^{-1}]-\mathbb{E}[r_i^{-1}]\mathbb{E}[r_j^{-1}]|\\
		&=\mathbb{E}[\tilde r_i^{-1}]|\mathbb{E}[\tilde r_j^{-1}]-\mathbb{E}[r_j^{-1}]|+
		\mathbb{E}[r_j^{-1}]|\mathbb{E}[\tilde r_i^{-1}]-\mathbb{E}[r_i^{-1}]|\\
		&\lesssim n^{-\frac{d-4}{2d}}.
	\end{split}
\end{equation}

\subsubsection{Upper bound for the difference between  \texorpdfstring{$\text{Cov}(\frac{1}{\tilde r_i},\frac{1}{r_j})$ and $\text{Cov}(\frac{1}{\tilde r_i},\frac{1}{r_j})$}{sdf}}

\begin{equation}
	\begin{split}
		&|\text{Cov}(\frac{1}{\tilde r_i},\frac{1}{r_j})-\text{Cov}(\frac{1}{\tilde r_i},\frac{1}{r_j})|\\
		&=|\mathbb{E}[\tilde r_i^{-1}]\mathbb{E}[\tilde r_j^{-1}]-\mathbb{E}[r_i^{-1}]\mathbb{E}[r_j^{-1}]
		+\mathbb{E}[\tilde r_i^{-1}]\mathbb{E}[\tilde r_j^{-1}]-\mathbb{E}[r_i^{-1}]\mathbb{E}[r_j^{-1}]|\\
		&\le|\mathbb{E}[\tilde r_i^{-1}]\mathbb{E}[\tilde r_j^{-1}]-\mathbb{E}[r_i^{-1}]\mathbb{E}[r_j^{-1}]|
		+|\mathbb{E}[\tilde r_i^{-1}]\mathbb{E}[\tilde r_j^{-1}]-\mathbb{E}[r_i^{-1}]\mathbb{E}[r_j^{-1}]|\\
	&\lesssim n^{- \frac{d-4}{2d}}.
	\end{split}
\end{equation}
\subsubsection{Estimate of \texorpdfstring{$\text{Cov}(\frac{1}{r_i},\frac{1}{r_j})$}{sdf}}

\begin{equation}
	\text{Cov}(\frac{1}{r_i},\frac{1}{r_j})\lesssim \text{Cov}(\frac{1}{\tilde r_i},\frac{1}{\tilde r_j}) +n^{-\frac{d-4}{2d}}\lesssim  n^{-\frac{d-4}{2d}}
\end{equation}

\subsubsection{Upper bound of Var\texorpdfstring{$(\frac{1}{n}\sum\limits_{i=1}^nr_i^{-1})$}{sdfsdf}}

\begin{equation}
	\begin{split}
		\text{Var$(\frac{1}{n}\sum\limits_{i=1}^nr_i^{-1})$}
		&\le \frac{1}{n^2}\sum\limits_{i=1}^n \text{Var}(r_i^{-1})+ \frac{1}{n^2}\sum\limits_{i=1,j=1,i\neq j}^n \text{Cov}(r_i^{-1},r_j^{-1})\\
		&\le \frac{1}{n^2}\sum\limits_{i=1}^n \mathbb{E}(r_i^{-2})+ \frac{1}{n^2}\sum\limits_{i=1,j=1,i\neq j}^n \text{Cov}(r_i^{-1},r_j^{-1})\\
	&\lesssim n^{-\frac{2}{d}-1}+ n^{-\frac{d-4}{2d}}\\
	&\lesssim n^{-\frac{d-4}{2d}}
	\end{split}
\end{equation}
\subsubsection{Final Step: Chebyshev's inequality}

By Chebyshev's inequality

\begin{equation}
	\mathbb{P}[\frac{1}{n}\sum\limits_{i=1}^nr_i^{-1}> A n^{\frac{1}{d}}+\mathbb{E}[\frac{1}{n}\sum\limits_{i=1}^nr_i^{-1}]]\le (A^2n^{\frac{2}{d}})^{-1}\text{Var$(\frac{1}{n}\sum\limits_{i=1}^nr_i^{-1})$}\lesssim A^{-2}n^{-\frac{1}{2}}
\end{equation}
since
\begin{equation}
	\mathbb{E}[\frac{1}{n}\sum\limits_{i=1}^nr_i^{-1}]=\mathbb{E}r_1^{-1}\lesssim n^\frac{1}{d}.
\end{equation}
This concludes the proof.

\section{Inequalities for Functions}

\subsection{Gagliardo-Nirenberg interpolation inequalities}

Here we quote the statements of Gagliardo-Nirenberg inequalities from \citep{leoni2017sobolev}. Note that here the term ``interpolation'' has nothing to do with our notion of interpolation.

\begin{mythm}[Gagliardo-Nirenberg interpolation for $\mathbb{R}^N$, general case, Theorem 12.87 in \cite{leoni2017sobolev}]

Let $1\le p,q\le\infty, m\in \mathbb{N},k\in \mathbb{N}_0$, with $0\le k<m$, and let $\theta$, $r$ be such that
\begin{equation}
	0\le \theta\le 1- k/m
\end{equation}
and
\begin{equation}
	(1- \theta)\left(\frac{1}{p}- \frac{m-k}{N}\right)+ \theta\left(\frac{1}{q}+ \frac{k}{N}\right)=\frac{1}{r}\in (-\infty,1].
\end{equation}

Then there exists a constant $c=c(m,N,p,q,\theta,k)>0$ such that
\begin{equation}
\label{gagliardo-nirenberg_interpolation1}
	|\nabla^k u|_r\le c\|u\|^ \theta_{L^q(\mathbb{R}^N)}\|\nabla^m u\|^{1- \theta}_{L^p(\mathbb{R}^N)}
\end{equation}
for every $u\in L^q(\mathbb{R}^N)\cap \dot{W}^{m,p}(\mathbb{R}^N)$, with the following exceptional cases:

\begin{enumerate}[(i)]
	\item If $k=0,mp<N,$ and $q=\infty$, we assume that $u$ vanishes at infinity.
	\item If $1<p<\infty$ and $m-k-N/p$ is a nonnegative integer, then (\ref{gagliardo-nirenberg_interpolation1}) only holds for $0< \theta\le 1-k/m$
\end{enumerate}

\end{mythm}

\begin{mythm}
	[Gagliardo-Nirenberg interpolation for domains, Theorem 13.61 in \cite{leoni2017sobolev}] Let $\Omega\subset \mathbb{R}^N$ be an open set with uniformly Lipschitz continuous boundary (with parameters $\epsilon$, L, M), let $0<l< \epsilon/(4(1+L))$, let $m,k\in \mathbb{N}$, with $m\ge 2$ and $1\le k< m$, and let $1\le p,q,r\le\infty$ be such that $p\le q$ and
	\begin{equation}
		\frac{k}{m}\frac{1}{p}+\left(1- \frac{k}{m}\right)\frac{1}{q}= \frac{1}{r}.
	\end{equation}

	If $p<q$, assume further that $\Omega$ is bounded. 

	Then for every $u\in L^q(\Omega)\cap \dot{W}^{m,p}(\Omega)$,
	\begin{equation}
		\|\nabla^k u\|_{L^r(\Omega)}\le cl^{-k}|\Omega|^{1/r- 1/q}\|u\|_{L^q(\Omega)}+ c\|u\|_{L^q(\Omega)}^{1-k/m}\|\nabla^m u\|_{L^p(\Omega)}^{k/m}
	\end{equation}
	if $p<q$, while
	\begin{equation}
	\label{gagliardo-nirenberg_interpolation2}
		\|\nabla^k u\|_{L^p(\Omega)}\le cl^{-k}\|u\|_{L^p(\Omega)}+c\|u\|_{L^p(\Omega)}^{1- k/m}\|\nabla^m u\|^{k/m}_{L^p(\Omega)}
	\end{equation}
	if $p=q$. Here, $c>0$ is a constant depending on $m, N, p, q$.
\end{mythm}
\begin{remark} Two remarks about notation:
\begin{itemize}
	\item the notation $|\cdot|_r$ is defined by
	\begin{equation}
		|u|_r:=\begin{cases}
			\|u\|_{L^r(\mathbb{R}^N)}& \text{if }r>0,\\
			\|\nabla^n u\|_{L^\infty(\mathbb{R}^N)}& \text{ if $r<0$ and $a=0$},\\
			|\nabla^nu|_{C^{0, a}(\mathbb{R}^N)}& \text{ if $r<0$ and $0<a<1$,}
		\end{cases}
	\end{equation}
	where if $r<0$ we set $n:=\text{floor}(-N/r)$ and $a:=-n-N/r\in [0,1)$, provided the right-hand sides are well-defined.
	\item $\dot{W}^{m,p}(\Omega)$ is the homogeneous Sobolev space and it coincides with the Sobolev space $W^{m,p}(\Omega)$ when $\Omega$ is a domain with finite measure.
\end{itemize}

\end{remark}
\begin{remark}

For our purposes, we need the inequality in two cases:
\begin{enumerate}[(i)]
	\item The domain is $\mathbb{R}^d$ with $d$ odd, $r=q=2$, $k=1,m=\frac{d+1}{2}, \theta=0$, then
	\begin{equation}
	1\times\left(\frac{1}{p}- \frac{\frac{d+1}{2}-1}{d}\right)+ 0\times\left(\frac{1}{2}+ \frac{1}{ d}\right)=\frac{1}{2}\in (-\infty,1].
\end{equation}
	which implies
	\begin{equation}
		p=2d
	\end{equation}

	Then
	\begin{equation}
		m-k-N/p=\frac{d+1}{2}-1- \frac{d}{2d}=\frac{d-2}{2}
	\end{equation}
	is not an integer because $d$ is odd. 

	Therefore, our case is not exceptional and from equation (\ref{gagliardo-nirenberg_interpolation1}), we get
	\begin{equation}
		\|Du\|_{L^{2d}(\mathbb{R}^d)}\le C_d\|D^{\frac{d+1}{2}}u\|_{L^2(\mathbb{R}^d)}
	\end{equation}
	\item The domain is $\Omega=\text{supp } \mathcal{P}=B(\mathbf{0},1)$, when $N=d$ is odd, $r=q=p=2,0\le k\le \frac{d+1}{2},m=\frac{d+1}{2}$, then
	\begin{equation}
		\frac{k}{m}\frac{1}{p}+\left(1- \frac{k}{m}\right)\frac{1}{q}= \frac{1}{r}
	\end{equation}
	holds.
	Then
	\begin{equation}
		\|D^ku\|_{L^{2}(\Omega)}\le C_{k,d}\|D^{\frac{d+1}{2}}u\|_{L^2(\Omega)}^ \alpha\|u\|^{1- \alpha}_{L^2(\Omega)}+C_{k,d}'\|u\|_{L^2(\Omega)}.
	\end{equation}

	Since
	\begin{equation}
		\|D^{\frac{d+1}{2}}u\|_{L^2(\Omega)}\le \|D^{\frac{d+1}{2}}u\|_{L^2(\mathbb{R}^d)},
	\end{equation}
	from equation (\ref{gagliardo-nirenberg_interpolation2}) we have
	\begin{equation}
	\label{laserjet}
		\|D^k u\|_{L^{2}(\Omega)}\le C_{k,d}\|D^{\frac{d+1}{2}}u\|_{L^2(\mathbb{R}^d)}^ \alpha\|u\|^{1- \alpha}_{L^2(\Omega)}+C_{k,d}'\|u\|_{L^2(\Omega)}.
	\end{equation}

	Note the theorem itself doesn't cover $k=0, \frac{d+1}{2}$ but equation (\ref{laserjet}) holds trivially in the two cases when $p=q=r$.
\end{enumerate}

\end{remark}
\clearpage
\subsection{Morrey's inequality}
\begin{mythm}[Morrey's inequality]
	
Suppose $u:\mathbb{R}^d\rightarrow \mathbb{R}$ has weak derivative $Du$ in $L^{2d}(\mathbb{R}^d)$
\begin{equation}
\label{morrey1}
	\sup\limits_{x\in \mathbb{R}^d,r>0}\frac{1}{\sqrt{r}}\left|u(x)-\dashint_{B(x,r)}u(y)dy\right|\le C_d\|Du\|_{L^{2d}(\mathbb{R}^d)}
\end{equation}

If in addition, $u\in L^q(\mathbb{R}^d)$, combining with Gagliardo-Nirenberg interpolation inequality for $\mathbb{R}^d$ (equation \ref{gagliardo-nirenberg_interpolation1}), we have
\begin{equation}
\label{morrey2}
	\sup\limits_{x\in \mathbb{R}^d,r>0}\frac{1}{\sqrt{r}}\left|u(x)-\dashint_{B(x,r)}u(y)dy\right|\le C_d\|D^{\frac{d+1}{2}}u\|_{L^2(\mathbb{R}^d)}
\end{equation}
\end{mythm}
\begin{remark}
	Here the notation $\dashint_{B(x,r)}$ means the average over the ball $B(x,r)$, i.e. $\frac{1}{|B(x,r)|}\int_{B(x,r)}$.
\end{remark}
\begin{remark}
	This version of Morrey's inequality is basically a middle step of Lemma 12.47 in \citep{leoni2017sobolev} (although it is a cube instead of a ball there) and the proof is simple enough to be written down below.
\end{remark}
\begin{proof} For any $x\in \mathbb{R}^d,r>0$

	\begin{align*}
			\left|u(x)-\dashint_{B(x,r)}u(y)dy\right|
			&= \left|\dashint_{B(x,r)}(u(x)-u(y))dy\right|\\
			&= \left|\dashint_{B(x,r)}\int_0^1 \frac{d}{dt}\Big(u(x)-u(x+t(y-x))\Big)dtdy\right|\\
			&\le  \dashint_{B(x,r)}\int_0^1 \|y-x\|\| D u(x+t(y-x))\|dtdy \\
			&=  \int_0^1 \left(\dashint_{B(x,r)}\|y-x\|\| D u(x+t(y-x))\|dy\right)dt\\
			&= \int_0^1 t^{-1 }\left(\dashint_{B(x,tr)}\|y-x\|\| D u(y)\|dy\right)dt\\
			&\le \int_0^1 t^{-1 }\left(\dashint_{B(x,tr)}\|y-x\|^ {\frac{2d}{2d-1}}dy\right)^ {\frac{2d-1}{2d }}\left(\dashint_{B(x,tr)} \|D u(y)\|^{2d}dy\right)^{\frac{1}{2d}}dt\\
			&\le O_d\left(\int_0^1 t^{-1 }\left( r^{\frac{2d}{2d-1}}t^{\frac{2d}{2d-1}}\right)^ {\frac{2d-1}{2d }}\left(r^{-d}t^{-d}\int_{\mathbb{R}^d} \|D u(y)\|^{2d}dy\right)^{\frac{1}{2d}}dt\right)\\
			&\le O_d\left(\sqrt{r}\|D u\|_{L^{2d}(\mathbb{R}^d)}\int_0^1 t^{- \frac{1}{2}}  dt\right)\\
			&= O_d\left(\sqrt{r}\|D u\|_{L^{2d}(\mathbb{R}^d)}\right)
	\end{align*}
\end{proof}
\subsection{Local H\"{o}lder Continuity around Samples}

	\begin{defn}[Measure of Local H\"{o}lder Continuity around Samples] For sample set $\mathcal{S}$ and index set $I\subset [n]$, we introduce the following measure of local H\"{o}lder continuity around samples
	\begin{equation}
		[f]_{\eta,\mathcal{S},I}=\sum\limits_{i\in I}\sup\limits_{x\in \mathbb{R}^d,r>0}\frac{1}{r}\left(f(x)\eta\left(\frac{x-X_i}{r_i}\right)-\dashint_{B(x,r)}f(y)\eta\left(\frac{y-X_i}{r_i}\right)dy\right)^2
	\end{equation}

	where $\eta(x)=\begin{cases}
		1,& \|x\|\le \frac{1}{4}\\
		e^{1- \frac{1}{2-4\|x\|}}, &\frac{1}{4}<\|x\|<\frac{1}{2}\\
		0, & \|x\|\ge \frac{1}{2}
	\end{cases}$
	
\end{defn}
\begin{lem} For any subset $I\subset [n], \beta\in (0,1)$ and $f\in L^2(\Omega)$
	\begin{equation}
		\|f\|_{L^2(\Omega)}^2\ge \frac{3}{4}\frac{\beta^d\pi^{\frac{d}{2}}}{2^d\Gamma(\frac{d}{2}+1)}\left(\sum_{i\in I}r_i^{d}f(X_i)^2
			-4\beta [f]_{\eta,\mathcal{S},I}\max\limits_{i\in I}r_i^{d+1}\right).
	\end{equation}
\label{tassadar}
\end{lem}

\begin{proof}
	We write
	\begin{align}
			\|f\|_{L^2(\Omega)}^2
			&\ge
			\sum_{i\in I}\int_{B(X_i,\beta r_i/2)}f(x)^2 d x\\
			&\ge
			\sum_{i\in I}\int_{B(X_i,\beta r_i/2)}f(x)^2 \eta\left(\frac{x-X_i}{ r_i}\right)^2 d x\\
			&\ge
			\sum_{i\in I} \frac{1}{|B(X_i,\beta r_i/2)|}\left(\int_{B(X_i,\beta r_i/2)}f(x) \eta\left(\frac{x-X_i}{r_i}\right) d x\right)^2.
	\end{align}
	Writing this expression as a normalized integral, we get
	\begin{align}
			&\sum_{i\in I} |B(X_i,\beta r_i/2)|\left(\dashint_{B(X_i,\beta r_i/2)}f(x) \eta\left(\frac{x-X_i}{r_i}\right) d x\right)^2\\
			&\geq
			\sum_{i\in I} |B(X_i,r_i/2)|\left(\frac{3}{4}f(X_i)^2-3\left(f(X_i)-\dashint_{B(X_i,\beta r_i/2)}f(x) \eta\left(\frac{x-X_i}{r_i}\right) d x\right)^2\right)\\
			&\ge
			\frac{3}{4}\sum_{i\in I}|B(X_i, \beta r_i/2)| f(X_i)^2
			-3[f]_{\eta,\mathcal{S},I}\sup\limits_{i\in I}\beta r_i B(X_i,\beta r_i/2)\\
			&=\frac{3}{4}\frac{\beta^d\pi^{\frac{d}{2}}}{2^d\Gamma(\frac{d}{2}+1)}\left(\sum_{i\in I}r_i^{d}f(X_i)^2
			-4\beta [f]_{\eta,\mathcal{S},I}\max\limits_{i\in I}r_i^{d+1}\right).
	\end{align}
\end{proof}
\begin{lem} For any subset $I\subset [n]$, we have
	\begin{equation}
		[f]_{\eta,\mathcal{S},I}\le O_d\Big(\big(1+\|f\|^2_{L^2(\Omega)}\big)\big(c^{d+1} \langle f\rangle_{\cH_c}+\max\limits_{i\in I}r_i^{-d-1}\big)\Big).
	\end{equation}
\label{zerutal}
\end{lem}
\begin{proof}

Define $\eta_i$ by
\begin{equation}
	\eta_i(x)=\eta\left(\frac{x-X_i}{r_i}\right)
\end{equation}
and
\begin{equation}
	A=\max \{c \langle f\rangle_{\cH_c}^{\frac{1}{d+1}},\max\limits_{i\in I} r_i^{-1}\}
\end{equation}

We prove our lemma by first proving the following inequalities:
	\begin{enumerate}[(a)]
		\item $[f]_{\eta,\mathcal{S},I}\le O_d\left(\sum\limits_{i\in I}\|D^{\frac{d+1}{2}}(f \eta_i)\|^2_{L^2(\mathbb{R}^d)}\right)$
		\item $\sum\limits_{i\in I}\|D^{\frac{d+1}{2}}(f \eta_i)\|_{L^2(\mathbb{R}^d)}^2
			\le
			O_d\left(
			\sum\limits_{j=0}^{\frac{d+1}{2}} A^{d+1-2j}\|D^jf\|_{L^2(\mathbb{R}^d)}^2
			\right)$
		\item $\|D^j f\|_{L^2(\mathbb{R}^d)}\le  O_d\left(\left(1+\|f\|_{L^2(\Omega)}\right)A^{j}\right)$
	\end{enumerate}
and then it follows that

\begin{equation}
	\begin{split}
		[f]_{\eta, \mathcal{S}, I}\le O_d\Big(\big(1+\|f\|^2_{L^2(\Omega)}\big)\big(c^{d+1} \langle f\rangle_{\cH_c}+\max\limits_{i\in I}r_i^{-d-1}\big)\Big).
	\end{split}
\end{equation}

\paragraph{Inequality (a).} This is a direct application of Morrey's inequality  (equation \eqref{morrey2}).
\paragraph{Inequality (b).}

Using Leibnitz rule we have

\begin{equation}
	\begin{split}
		\|D^{\frac{d+1}{2}}(f\eta_i)\|_{L^2(\mathbb{R}^d)}^2
		&\le O_d\left(\sum_{|\alpha|= \frac{d+1}{2}}\sum_{0\le \beta\le \alpha}\|D^{\alpha- \beta} \eta_iD^ {\beta}f\|_{L^2(\mathbb{R}^d)}^2\right).
	\end{split}
\end{equation}

Since the function $D^{\alpha- \beta} \eta_iD^ {\beta}f$ is supported within the ball $B(X_i,r_i)$, we have
\begin{equation}
	\begin{split}
		\|D^{\frac{d+1}{2}}(f\eta_i)\|_{L^2(\mathbb{R}^d)}^2
		&=O_d\left(\sum_{|\alpha|= \frac{d+1}{2}}\sum_{0\le \beta\le \alpha}\|D^{\alpha- \beta} \eta_iD^ {\beta}f\|_{L^2(B(X_i,r_i))}^2\right).
	\end{split}
\end{equation}

By H\"{o}lder inequality,

\begin{equation}
	\begin{split}
		\|D^{\frac{d+1}{2}}(f\eta_i)\|_{L^2(\mathbb{R}^d)}^2
		&\le O_d\left(\sum_{|\alpha|= \frac{d+1}{2}}\sum_{0\le \beta\le \alpha}   \|D^{\alpha- \beta} \eta_i\|_{L^\infty(B(X_i,r_i))}^2\|D^{\beta}f\|_{L^2(B(X_i,r_i))}^2\right).
	\end{split}
\end{equation}

Using the fact that
\begin{equation}
	\|D^ \beta \eta_i\|_{L^\infty(\mathbb{R}^d)}\le C_d r_i^{- |\beta|},
\end{equation}
 we then get
\begin{equation}
	\begin{split}
		\|D^{\frac{d+1}{2}}(f\eta_i)\|_{L^2(\mathbb{R}^d)}^2
		&= O_d\left(\sum_{|\alpha|= \frac{d+1}{2}}\sum_{0\le \beta\le \alpha}\frac{\|D^{ \beta}f\|_{L^2(B(X_i,r_i))}^2}{r_i^{2|\alpha- \beta|}}\right)\\
		&\le O_d\left(\sum_{j=0}^{\frac{d+1}{2}}\frac{\|D^jf\|_{L^2(B(X_i,r_i))}^2}{r_i^{d+1-2j}}\right)\\
		&\le O_d\left(\sum_{j=0}^{\frac{d+1}{2}}\frac{\|D^jf\|_{L^2(B(X_i,r_i))}^2}{\min\limits_{i\in I}r_i^{d+1-2j}}\right).
	\end{split}
\end{equation}

Then we have

\begin{equation}
	\begin{split}
		\sum\limits_{i\in I}\|D^{\frac{d+1}{2}}(f \eta_i)\|_{L^2(\mathbb{R}^d)}^2
		&\le O_d\left(\sum_{j=0}^{\frac{d+1}{2}}\frac{\sum\limits_{i\in I}\|D^jf\|_{L^2(B(X_i,r_i))}^2}{\min\limits_{i\in I}r_i^{d+1-2j}}\right)\\
		&\le O_d\left(\sum_{j=0}^{\frac{d+1}{2}}\frac{\|D^jf\|_{L^2(\Omega)}^2}{\min\limits_{i\in I}r_i^{d+1-2j}}\right)\\
		&\le O_d\left(
			\sum\limits_{j=0}^{\frac{d+1}{2}} A^{d+1-2j}\|D^jf\|_{L^2(\Omega)}^2
			\right).
	\end{split}
\end{equation}

\paragraph{Inequality (c).}

Here use Gagliardo-Nirenberg interpolation inequality for domains (equation \eqref{laserjet}) and the fact \begin{equation}
	\|D^{\frac{d+1}{2}} f\|_{L^2(\Omega)}^2\le\|D^{\frac{d+1}{2}} f\|_{L^2(\mathbb{R}^d)}^2\le c^{d+1}\langle f\rangle_{\cH_c},
\end{equation}
we have

\begin{equation}
	\begin{split}
		\|D^j f\|_{L^2(\Omega)}
		&\le  O_d\left(\|D^{\frac{d+1}{2}} f\|_{L^2(\Omega)}^{\frac{2 j}{d+1}}\|f\|^{1-\frac{2 j}{d+1}}_{L^2(\Omega)}+\|f\|_{L^2(\Omega)}\right)\\
		&\le O_d\left( c^{j}\langle f\rangle_{\cH_c}^{\frac{j}{d+1}}\|f\|^{1-\frac{2 j}{d+1}}_{L^2(\Omega)}+\|f\|_{L^2(\Omega)}\right)\\
		&\le O_d\left(\left(1+\|f\|_{L^2(\Omega)}\right)A^{j}\right).
	\end{split}
\end{equation}

\end{proof}
\begin{prop}
\label{Holder}
	For any subset $I\subset [n] $ and $f\in L^2(\Omega)$, we have
	\begin{equation}
		\begin{split}
			&\|f\|_{L^2(\Omega)}^2\ge\min \left\{1, \Omega_d\left(\left(\frac{\min\limits_{i\in I}r_i^{-d-1}\sum_{i\in I}r_i^{d}f(X_i)^2}{\max\limits_{i\in I}r_i^{-d-1}+c^{d+1} \|f\|_{\cH_c}}\right)^d\sum_{i\in I}r_i^{d}f(X_i)^2\right)\right\}.
		\end{split}
	\end{equation}
\end{prop}

\begin{proof}
	Without loss of generality suppose that $\|f\|^2_{L^2(\Omega)}\le 1$. Then from Lemma \ref{zerutal}, there is a constant $C_d$ such that
	\begin{equation}
		[f]_{\eta, \mathcal{S}, I}\le C_d\left(c^{d+1} \|f\|_{\cH_c}+\max\limits_{i\in I}r_i^{-d-1}\right).
	\end{equation}

	From Lemma \ref{tassadar}, we have for any $\beta\in (0,1)$:
	
	\begin{equation}
		\begin{split}
			\|f\|_{L^2(\Omega)}^2&\ge \frac{3}{4}\frac{\beta^d\pi^{\frac{d}{2}}}{2^d\Gamma(\frac{d}{2}+1)}\left(\sum_{i\in I}r_i^{d}f(X_i)^2
			-4\beta [f]_{\eta,\mathcal{S},I}\max\limits_{i\in I}r_i^{d+1}\right)\\
			&\ge \frac{3}{4}\frac{\beta^d\pi^{\frac{d}{2}}}{2^d\Gamma(\frac{d}{2}+1)}\left(\sum_{i\in I}r_i^{d}f(X_i)^2
			-4\beta C_d\max\limits_{i\in I}r_i^{d+1}\left(c^{d+1} \|f\|_{\cH_c}+\max\limits_{i\in I}r_i^{-d-1}\right)\right).
	\end{split}
	\end{equation}

	Taking \begin{equation}
		\beta= \frac{\max\limits_{i\in I}r_i^{-d-1}\sum_{i\in I}r_i^{d}f(X_i)^2}{8C_d\left(c^{d+1} \|f\|_{\cH_c}+\max\limits_{i\in I}r_i^{-d-1}\right)},
	\end{equation}
	we get
	\begin{equation}
		\begin{split}
			\|f\|_{L^2(\Omega)}^2 &\ge \frac{3}{4}\frac{\beta^d\pi^{\frac{d}{2}}}{2^d\Gamma(\frac{d}{2}+1)}   \left(\sum_{i\in I}r_i^{d}f(X_i)^2 -\frac{1}{2}\sum_{i\in I}r_i^{d}f(X_i)^2\right)\\
			&\ge\frac{3}{8}\frac{ \pi^{\frac{d}{2}}}{2^d\Gamma(\frac{d}{2}+1)}\left(\frac{\min\limits_{i\in I}r_i^{-d-1}\sum_{i\in I}r_i^{d}f(X_i)^2}{8C_d\left(c^{d+1} \|f\|_{\cH_c}+\max\limits_{i\in I}r_i^{-d-1}\right)}\right)^d \sum_{i\in I}r_i^{d}f(X_i)^2  \\
			&\ge  \Omega_d\left(\left(\frac{\min\limits_{i\in I}r_i^{-d-1}\sum_{i\in I}r_i^{d}f(X_i)^2}{ c^{d+1} \|f\|_{\cH_c}+\max\limits_{i\in I}r_i^{-d-1} }\right)^d \sum_{i\in I}r_i^{d}f(X_i)^2 \right).
		\end{split}
	\end{equation}

\end{proof}

\subsection{Upper Bound of \texorpdfstring{$\langle\algo_{c}\rangle_{\mathcal{H}_c}$}{sdf}}

\begin{prop}
\label{interpolation1}
	With probability at least $1-O_{d,\rho}(\frac{1}{\sqrt{n}})$, for any $c>0$ there is a function $g$ interpolating $\mathcal{S}$ such that
	\begin{equation}
		\langle g\rangle_{\cH_c}\le
		\frac{1}{3}\|f_0\|_{L^2(\omega)}^2+O_{d, \rho, f_0}\left(\frac{\sqrt[d]{n}}{c}\left(1+\frac{\sqrt[d]{n}}{c}\right)^{d}\right)
	\end{equation}

	Since $\algo_{c}$ has the smallest RKHS norm among all interpolating functions, we have
	\begin{equation}
	\label{interpolation_eq}
		\langle \algo_{c}\rangle_{\cH_c}\le
		\frac{1}{3}\|f_0\|_{L^2(\omega)}^2+O_{d, \rho, f_0}\left(\frac{\sqrt[d]{n}}{c}\left(1+\frac{\sqrt[d]{n}}{c}\right)^{d}\right)
	\end{equation}
\end{prop}
\begin{proof}
Define $r_i=\min\limits_{j\neq i}\|X_i-X_j\|$ and 
\begin{equation}
	\eta(x)=\begin{cases}
		1,& \|x\|\le \frac{1}{4}\\
		e^{1- \frac{1}{2-4\|x\|}}, &\frac{1}{4}<\|x\|<\frac{1}{2}\\
		0, & \|x\|\ge \frac{1}{2}
	\end{cases}
\end{equation}

and for $\alpha\in (0,\frac{1}{2})$ take
\begin{equation}
	g_\alpha(x):=\sum\limits_{i=1}^n Y_i\eta\left(\frac{x-X_i}{\alpha r_i}\right).
\end{equation}

First,

\begin{equation}
	\begin{split}
		\|g_ \alpha\|^2_{L^2(\mathbb{R}^d)}&= \sum_iY_i^2\|\eta_{X_i, \alpha r_i}\|_{L^2(\mathbb{R}^d)}^2\\
		&=\alpha^d \|\eta\|_{L^2(\mathbb{R}^d)}^2\sum_iY_i^2 r_i^d\\
		&\le \alpha^d \|\eta\|_{L^2(\mathbb{R}^d)}^2\sum_i (\|f_0\|_{L^\infty(\Omega)}+1)^2 r_i^d\\
		&\le \alpha^d \|\eta\|_{L^2(\mathbb{R}^d)}^2 (\|f_0\|_{L^\infty(\Omega)}+1)^2 \sum_ir_i^d\\
		&\le \frac{2^d|\Omega|}{|B_d(1)|}\alpha^d \|\eta\|_{L^2(\mathbb{R}^d)}^2 (\|f_0\|_{L^\infty(\Omega)}+1)^2 \\
		&\le O_d(\alpha^d)
	\end{split}
\end{equation}
Therefore, we can take $\alpha$ to be a constant dependent only on $d$ and $f_0$ such that
\begin{equation}
		\|g_ \alpha(x)\|_{L^2(\mathbb{R}^d)}^2\le \frac{1}{3}\|f_0\|_{L^2(\Omega)}^2
		\end{equation}

Since 
\begin{equation}
	\langle \eta\left(\frac{x-X_i}{\alpha r_i}\right)\rangle_{k}= \alpha^{d-2k}r_i^{d-2k} \langle \eta\rangle_{k}
\end{equation}

and

\begin{equation}
	\langle u,v\rangle_{\mathcal{H}_c}=0, \text{ if }\text{supp }u\cap \text{supp }v=\emptyset
\end{equation}

then for $k\in \mathbb{N}$ we have
\begin{equation}
	\langle g\rangle_{\mathcal{H}_c}=\|g\|_{L^2(\mathbb{R}^d)}^2+
	\sum\limits_{i=1}^n Y_i^2(\alpha r_i)^{d-2k} \langle \eta\rangle_{k}.
\end{equation}

So when $d$ is odd,

\begin{equation}
	\begin{split}
		\langle g \rangle_{\mathcal{H}_c}&=
	\|g\|_{L^2(\mathbb{R}^d)}^2+
	\sum\limits_{k=1}^{\frac{d+1}{2}}\sum\limits_{i=1}^n \binom{\frac{d+1}{2}}{k}Y_i^2c^{-2k}(\alpha r_i)^{d-2k} \langle \eta\rangle_{k}\\
	&\le
	\frac{1}{3}\|f_0\|_{L^2(\mathbb{R}^d)}^2+
	\sum\limits_{k=1}^{\frac{d+1}{2}}\sum\limits_{i=1}^n \binom{\frac{d+1}{2}}{k}\left(\|f_0\|_{L^\infty(\Omega)}+1\right)^2c^{-2k}(\alpha r_i)^{d-2k} \langle \eta\rangle_{k}\\
	&\le 
	\frac{1}{3}\|f_0\|_{L^2(\mathbb{R}^d)}^2+
	O_{d, \rho}\left(
	\left(\|f_0\|_{L^\infty(\Omega)}+1\right)^2\sum\limits_{k=1}^{\frac{d+1}{2}}\sum\limits_{i=1}^n c^{-2k}(\alpha r_i)^{d-2k}
	\right)\\
	&\le 
	\frac{1}{3}\|f_0\|_{L^2(\mathbb{R}^d)}^2+
	O_{d, \rho,f_0}\left(
	\sum\limits_{k=1}^{\frac{d+1}{2}}\sum\limits_{i=1}^n  c^{-2k}r_i ^{d-2k}
	\right).
	\end{split}
\end{equation}

From Proposition \ref{ri1}, with probability at least $1-O_{d, \rho}\left(\frac{1}{\sqrt{n}}\right)$ we have
\begin{equation}
	\sum\limits_{i=1}^n  r_i^{d-2k}\le O_{d, \rho}\left(n^{2k/d}\right).
\end{equation}

Then with the same probability,

\begin{equation}
	\begin{split}
		\langle g\rangle_{\cH_c}\le \frac{1}{3}\|f_0\|_{L^2(\omega)}^2+O_{d, \rho, f_0}\left(\frac{\sqrt[d]{n}}{c}\left(1+\frac{\sqrt[d]{n}}{c}\right)^{d}\right).
	\end{split}
\end{equation}

\end{proof}

\bibliography{refs}

\newpage

\appendix

\end{document}